\documentclass{article}

\usepackage[final, nonatbib]{neurips_2019}

\usepackage[utf8]{inputenc} 
\usepackage[T1]{fontenc}    
\usepackage[colorlinks=True, citecolor=blue]{hyperref}       
\usepackage{url}            
\usepackage{booktabs}       
\usepackage{amsfonts}       
\usepackage{nicefrac}       
\usepackage{microtype}      

\usepackage{amsmath, amssymb, amsthm, mathrsfs}
\usepackage{graphicx}
\usepackage{booktabs, multirow}
\usepackage{caption, subcaption}
\usepackage[numbers]{natbib}

\theoremstyle{plain}
\newtheorem{theorem}{Theorem}[section]
\newtheorem{lemma}[theorem]{Lemma}
\theoremstyle{definition}
\newtheorem*{definition}{Definition}
\theoremstyle{remark}
\newtheorem*{remark}{Remark}

\numberwithin{equation}{section}

\renewcommand{\intercal}{\mathsf{T}}
\newcommand{\bx}{{\boldsymbol{x}}}
\newcommand{\by}{{\boldsymbol{y}}}
\newcommand{\bz}{{\boldsymbol{z}}}
\newcommand{\R}{\mathbb{R}}
\newcommand{\cM}{\mathcal{M}}
\newcommand{\cN}{\mathcal{N}}
\newcommand{\cR}{\mathcal{R}}
\newcommand{\poly}{\mathrm{poly}}
\newcommand{\tr}{\mathrm{tr}}
\newcommand{\eqdef}{\stackrel{\mathrm{def}}{=}}

\title{%
  Global Convergence of Gradient Descent \\
  for Deep Linear Residual Networks
}

\author{%
  Lei Wu\thanks{Equal contribution} \quad Qingcan Wang\footnotemark[1]
  \quad Chao Ma \\
  Program in Applied and Computational Mathematics\\
  Princeton University\\
  Princeton, NJ 08544, USA\\
  \texttt{\{leiwu,qingcanw,chaom\}@princeton.edu}
}

\begin{document}

\maketitle

\begin{abstract}
  We analyze the global convergence of gradient descent for deep linear residual
  networks by proposing a new initialization: zero-asymmetric (ZAS)
  initialization. It is motivated by avoiding stable manifolds of saddle points.
  We prove that under the ZAS initialization, for an arbitrary target matrix,
  gradient descent converges to an $\varepsilon$-optimal point in $O\left( L^3
  \log(1/\varepsilon) \right)$ iterations, which scales polynomially with the
  network depth $L$. Our result and the $\exp(\Omega(L))$ convergence time for the 
  standard initialization (Xavier or near-identity)
  \cite{shamir2018exponential} together demonstrate the importance of the
  residual structure and the initialization in the optimization for deep linear
  neural networks, especially when $L$ is large.
\end{abstract}

\section{Introduction}

It is widely observed that simple gradient-based optimization algorithms are
efficient for training deep neural networks \cite{zhang2016understanding}, whose
landscape is highly non-convex. To explain the efficiency, traditional
optimization theories cannot be directly applied and the special structures of
neural networks must be taken into consideration. Recently many researches are
devoted to this topic \cite{kawaguchi2016deep, zhang2016understanding,
bartlett2018gradient, du2018gradient, du2018deepgradient, allen2019convergence,
zou2018stochastic, song2018mean, rotskoff2018neural}, but
the theoretical understanding is still far from sufficient.

In this paper, we focus on a simplified case: the deep linear neural network
\begin{equation}
  f(\bx; W_1, \dots, W_L) = W_L W_{L-1} \cdots W_1 \bx,
\end{equation}
where $W_1, \dots, W_L$ are the weight matrices and $L$ is the depth. Linear
networks are simple since they can only represent linear transformation, but
they preserve one of the most important aspects of deep neural networks, the
layered structure. Therefore, analysis of linear networks will be helpful for
understanding nonlinear cases. For example, the random orthogonal initialization
proposed in \cite{saxe2013exact} that analyzes the gradient descent dynamics of
deep linear networks was later shown to be useful for training recurrent
networks with long term dependences \cite{vorontsov2017orthogonality}.

Despite the simplicity, the optimization of deep linear neural networks is still
far from being well understood, especially the global convergence.
\cite{shamir2018exponential} proves that the number of iterations required for
convergence could scales exponentially with the depth $L$. The result requires
two conditions: (1) the width of each layer is $1$; (2) the gradient descent
starts from the standard Xavier \cite{glorot2010understanding} or near-identity
\cite{he2016deep} initialization. If these conditions break, the negative
results does not imply that gradient descent cannot efficiently learn deep
linear networks in general.  \cite{du2019width} shows that if the width of every
layer increases with the network depth, gradient descent with the Gaussian
random initialization does find the global minima while the convergence time
only scales polynomially with the depth. Here we attempt to circumvent the
negative result in \cite{shamir2018exponential} by using better initialization
strategies instead of increasing the width.

\paragraph{Our Contributions} 
We propose the \emph{zero-asymmetric (ZAS) initialization}, which initializes
the output layer $W_L$ to be zero and all the other layers $W_l$, $l = 1, \dots,
L-1$ to be identity. So it is a \emph{linear residual network} with all the
residual blocks and the output layer being zero. We then analyze how the
initialization affects the gradient descent dynamics.

\begin{itemize}
  \item We prove that starting from the ZAS initialization, the number of
    iterations required for gradient descent to find an $\varepsilon$-optimal
    point is $O\left(L^3 \log(1/\varepsilon)\right)$. The only requirement for
    the network is that the width of each layer is not less than the input
    dimension and the result applies to arbitrary target matrices.
  \item We numerically compare the gradient descent dynamics between the ZAS and
    the near-identity initialization for multi-dimensional deep linear
    networks. The comparison clearly shows that the convergence of gradient
    descent with the near-identity initialization involves a saddle point escape
    process, while the ZAS initialization never encounters any saddle point
    during the whole optimization process. 
  \item We provide an extension of the ZAS initialization to the nonlinear case.
    Moreover, the numerical experiments justify its superiority compared to the
    standard initializations. 
\end{itemize}

\subsection{Related work}

\paragraph{Linear networks}
The first line of works analyze the whole landscape. The early work
\cite{baldi1989neural} proves that for two-layer linear networks, all the local
minima are also global minima, and this result is extended to deep linear
networks in \cite{kawaguchi2016deep, laurent2018deep}. \cite{hardt2016identity}
provides a simpler proof of this result for deep residual networks, and shows
that the Polyak-\L{}ojasiewicz condition is satisfied in a neighborhood of a
global minimum. However, these results do not imply that gradient descent can
find global minima, and also cannot tell us the number of iterations required
for convergence.

The second line of works directly deal with the trajectory of gradient descent
dynamics, and our work lies in this venue. \cite{saxe2013exact} provides an
analytic analysis to the gradient descent dynamics of linear networks, which
nevertheless does not show that gradient descent can find global minima.
\cite{ji2018gradient} studies the properties of solutions that the gradient
descent converges to, without providing any convergence rate.
\cite{bartlett2018gradient, arora2018convergence} consider the following
simplified objective function for whitened data,
\[
  \cR(W_1, \dots ,W_L) = \frac{1}{2} \|W_L \cdots W_1 - \Phi\|_F^2.
\]

Specifically, \cite{bartlett2018gradient} analyzes the convergence of gradient
descent with the identity initialization: $W_L = \cdots = W_1 = I$, and proves
that if the target matrix $\Phi$ is positive semi-definite or the initial loss
is small enough, a polynomial-time convergence can be guaranteed.
\cite{arora2018convergence} extends the analysis to more general target matrices
by imposing more conditions on the initialization: (1) approximately balance
condition, $\|W_{l+1}^\intercal W_{l+1} - W_l W_l^\intercal\|_F \le \delta$; (2)
rank-deficient condition, $\|W_L \cdots W_1 - \Phi\|_F \le \sigma_{\min}(\Phi) -
c$ for a constant $c > 0$. The condition (2) still requires small initial loss,
thus the convergence is local in nature. As a comparison, we do not impose any
assumption on the target matrix or the initial loss. 

As mentioned above, our work is closely related to \cite{shamir2018exponential},
which proves that for one-dimensional deep linear networks, gradient descent
with the standard Xavier or near-identity initialization requires at least
$\exp(\Omega(L))$ iterations for fitting the target matrix $\Phi = -I$. However,
our result shows that this difficulty can be overcome by adopting a better
initialization. \cite{du2019width} shows that if the width of each layer is
larger than $\Omega(L\log(L))$, then gradient descent converges to global minima
at a rate $O(\log(1/\varepsilon))$. As a comparison, our result only requires
that the width of each layer is not less than the input dimension. 

\paragraph{Nonlinear networks}
\cite{du2018deepgradient, allen2019convergence, zou2018stochastic} establish the
global convergence for deep networks with the width $m \ge \poly(n,L)$, where
$n$ denotes the number of training examples. \cite{e2019analysis} proves a
similar result but for specific neural networks with long-distance skip
connections, which only requires the depth $L \ge \poly(n)$ and the width $m \ge
d+1$, where $d$ is the input dimension.

The ZAS initialization we propose also closely resembles the ``fixup
initialization'' recently proposed in \cite{zhang2018residual}. Therefore, our
result partially provides a theoretical explanation to the efficiency of fixup
initialization for training deep residual networks.

\section{Preliminaries}

Given training data ${\{(\bx_i, \by_i)\}}_{i=1}^n$ where $\bx_i \in \R^{d_\bx}$
and $\by_i \in \R^{d_\by}$, a linear neural network with $L$ layers is defined
as 
\begin{equation}
  f(\bx; W_1, \dots, W_L) = W_L W_{L-1} \cdots W_1 \bx,
\end{equation}
where $W_l \in \R^{d_l \times d_{l-1}}$, $l = 1, \dots, L$ are parameter
matrices, and $d_0 = d_\bx$, $d_L = d_\by$. Then the least-squares loss
\begin{equation}
  \tilde\cR(W_1, \dots, W_L) \eqdef
  \frac{1}{2} \|W_L W_{L-1} \cdots W_2 W_1 X - Y\|_F^2,
  \label{eqn:loss_x}
\end{equation}
where $X = (\bx_1, \bx_2, \dots, \bx_n) \in \R^{d_\bx \times n}$ and $Y =
(\by_1, \by_2, \dots, \by_n) \in \R^{d_\by \times n}$.

Following \cite{bartlett2018gradient, arora2018convergence}, in this paper we
focus on the following simplified objective function
\begin{equation}
  \cR(W_1, \dots, W_L) \eqdef
  \frac{1}{2} \|W_L W_{L-1} \cdots W_2 W_1 - \Phi\|_F^2,
  \label{eqn:loss}
\end{equation}
where $W_l \in \R^{d \times d}$, $l = 1, \dots, L$ and $\Phi \in \R^{d \times
d}$ is the target matrix. Here we assume $d_l = d$, $l = 1, \dots, L$ for
simplicity.

The gradient descent is given by
\begin{equation}
  W_l(t+1) = W_l(t) - \eta \nabla_l \cR(t),
  \quad l = 1, \dots, L,\ t = 0, 1, 2, \dots
  \label{eqn:gd_discrete}
\end{equation}
In the following, we will always use the index $t$ to denote the value of a
variable after the $t$-th iteration. $\nabla_l \cR$ is the gradient of $\cR$
with respect to the weight matrix $W_l$:
\[
  \nabla_l \cR \eqdef \frac{\partial \cR}{\partial W_l}
  = W^\intercal_{L:l+1} (W_{L:1} - \Phi) W^\intercal_{l-1:1},
\]
where $W_{l_2:l_1} \eqdef W_{l_2} W_{l_2-1} \cdots W_{l_1+1} W_{l_1}$. Moreover,
we keep the learning rate $\eta > 0$ fixed for all iterations.

\paragraph{Notations}
In matrix equations, let $I$ and 0 be the $d$-dimensional identity matrix and
zero matrix respectively. Let $\lambda_{\min}(S)$ be the minimal eigenvalue of a
symmetric matrix $S$ and $\sigma_{\min}(A)$ be the minimal singular value of a
square matrix $A$. Let $\|A\|_F$ and $\|A\|_2$ be the Frobenius norm and
$\ell_2$ norm of matrix $A$ respectively. Recall that $A(t)$ denotes the value
of any variable $A$ after the $t$-th iteration, and $\nabla_l \cR$ is the
gradient of $\cR$ with respect to the weight matrix $W_l$. We use standard
notation $O(\cdot)$ and $\Omega(\cdot)$ to hide constants independent of network
depth $L$.

\section{Zero-asymmetric initialization}

In this section, we first describe the zero-asymmetric initialization, and then
illustrate by a simple example why this special initialization is helpful for
optimization.

\begin{definition}
  For deep linear neural network \eqref{eqn:loss}, define the
  \emph{zero-asymmetric (ZAS) initialization} as
  \begin{equation}
    W_l(0) = I,\ l = 1, \dots, L-1, \quad \text{and} \quad W_L(0) = 0.
  \label{eqn:init}
  \end{equation}
\end{definition}

Under the ZAS initialization, the function represented by the network is a zero
matrix. While commonly used initialization such as the Xavier and the
near-identity initialization treats all the layers equally, our initialization
takes the output layer differently. In this sense, we call the initialization
asymmetric.

Let $W_l = I + U_l$, $l = 1, \dots, L-1$, then the linear network has the
residual form
\[
  \cR = \frac{1}{2} \|W_L (I + U_{L-1}) \cdots (I + U_1) - \Phi\|_F^2.
\]
Since $\partial \cR / \partial U_l = \partial \cR / \partial W_l$, the dynamics
will be the same as ZAS if we initialize $U_l(0) = W_L(0) = 0$. Therefore, ZAS
is equivalent to initializing all the residual blocks and the output layer with
zero in a linear residual network. From this perspective, the ZAS initialization
closely resembles the ``fixup initialization'' \cite{zhang2018residual} for
nonlinear ResNets.

\paragraph{Understanding the role of initialization}
Following \cite{shamir2018exponential}, consider the following optimization
problem for one-dimensional linear network with target $\Phi = -1$:
\begin{equation}
  \cR(w_1,w_2,\dots,w_L) = (w_L w_{L-1} \cdots w_1 + 1)^2 / 2.
\end{equation}
The origin $O(0, \dots, 0)$ is a saddle point of $\cR$, so gradient descent with
small initialization, e.g., Xavier initialization, will spend long time
escaping the neighborhood of $O$. In addition,
\[
  \cM = \{(w_1, \dots, w_L): w_1 = w_2 = \cdots = w_L \ge 0\}
\]
is a stable manifold of $O$, i.e., gradient flow starting from any point in
$\cM$ will converge to $O$.  The near-identity initialization introduces
perturbation to leave $\cM$: $w_l(0) \sim \cN\left( 1, \sigma^2 \right)$, $l =
1, \dots, L$ for some small $\sigma$. However, \cite{shamir2018exponential}
proves that it will still be attracted to the neighborhood of $O$, thus cannot
guarantee the polynomial-time converge. As a comparison, the ZAS initialization
breaks the symmetry by initialize the output layer to be 0.

Figure~\ref{fig:illustration} provides a numerical result for depth $L = 2$.
The near-identity initialization (blue curve) spends long time escaping the
saddle region, while the ZAS initialization (red curve) converges to the global
minima without attraction by the saddle point.

\begin{figure}[t]
  \centering
  \includegraphics[width=.4\textwidth]{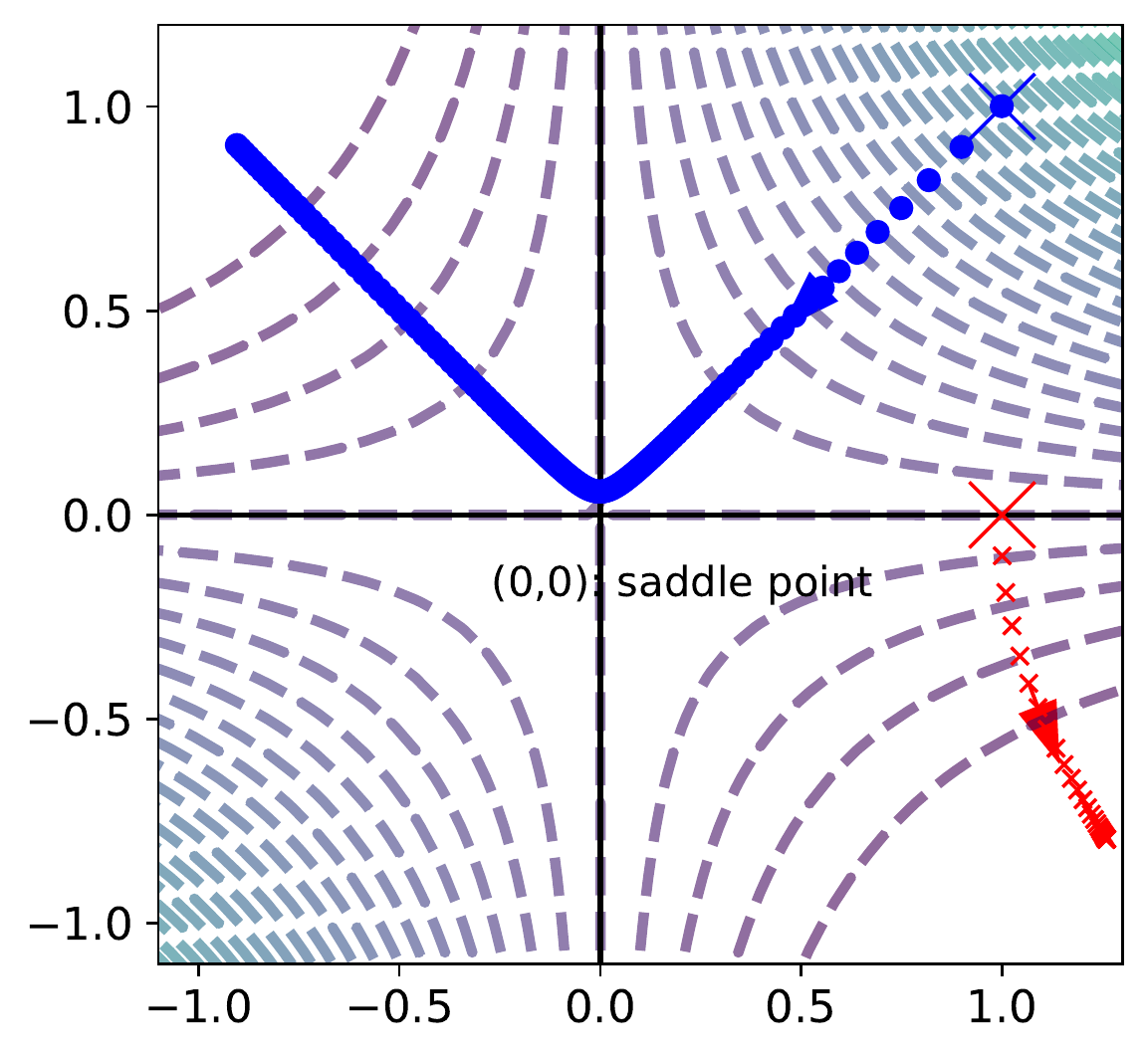}
  \qquad
  \includegraphics[width=.4\textwidth]{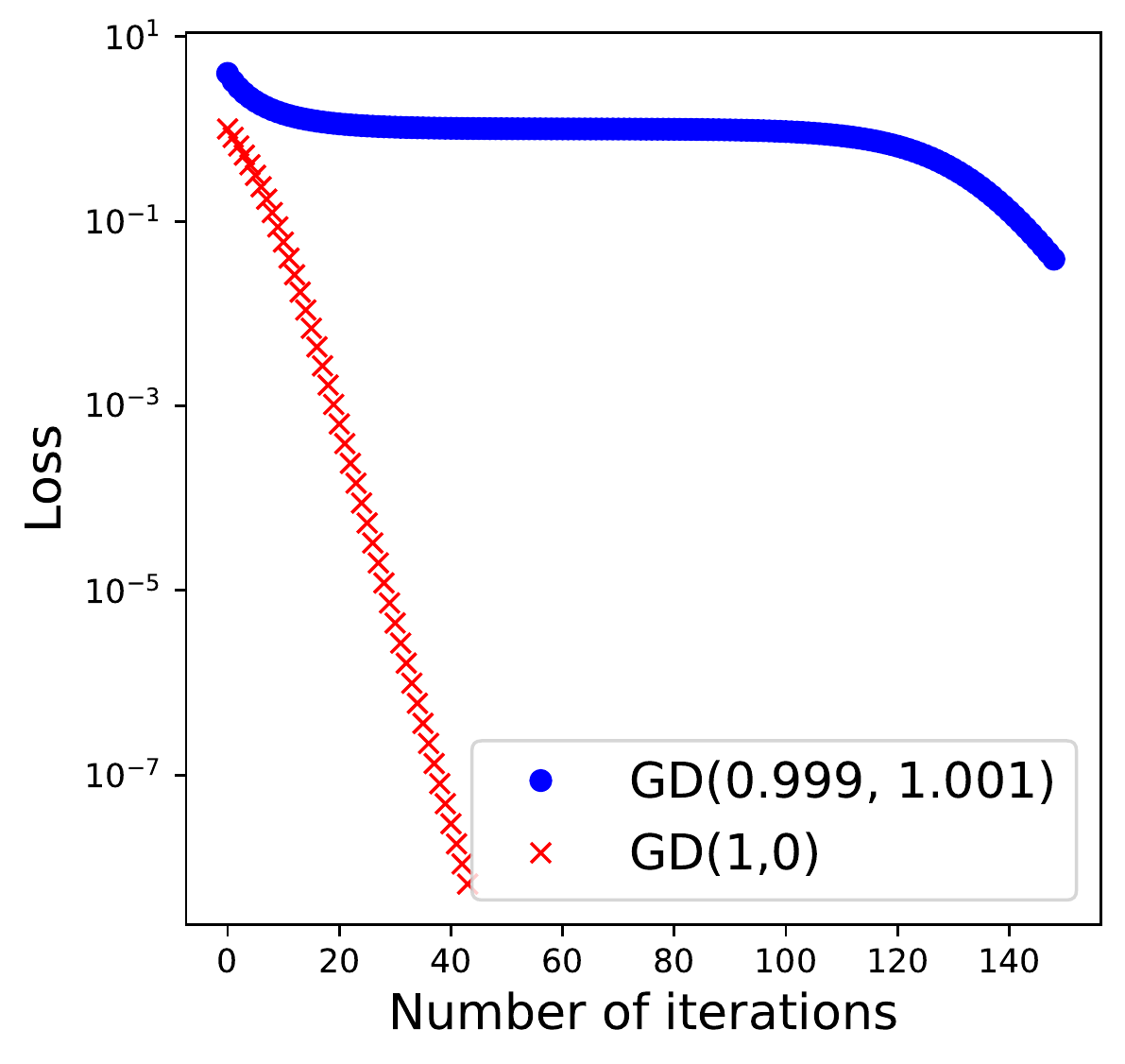}
  \caption{%
    \textbf{Left:} The landscape of the toy model $\cR(w_1, w_2)$ and the two
    gradient descent trajectories. \textbf{Right:} The dynamics of loss for two
    gradient descent trajectories. The blue curve is the gradient descent
    trajectory initialized from $(1 - 0.001, 1 + 0.001)$ (near-identity), and
    the red curve corresponds to the ZAS initialization $(1, 0)$. We observe
    that the blue curve takes a long time in the neighborhood of saddle point
    $(0, 0)$, however the red curve does not. 
  }%
  \label{fig:illustration}
\end{figure}

\section{Main results}%
\label{sec:main}

We first provide and prove the continuous version of our main convergence
result, i.e., the limit dynamics when $\eta \to 0$. Then we give the result for
discrete gradient descent, whose detailed proof is left to the appendix. 

\subsection{Continuous-time gradient descent} 

The continuous-time gradient descent dynamics is given by
\begin{equation}
  \dot W_l(t) = - \nabla_l \cR(t), \quad l = 1, \dots, L,\ t \ge 0.
  \label{eqn:gd_cont}
\end{equation}
In this section, we always denote $\dot{A}(t) = dA(t) / dt$ for any variable $A$
depending on $t$. For the continuous dynamics, we have the following convergence
result.

\begin{theorem}[Continuous-time gradient descent]\label{thm:cont}
  For the deep linear network \eqref{eqn:loss}, the continuous-time gradient
  descent \eqref{eqn:gd_cont} with the zero-asymmetric initialization
  \eqref{eqn:init} satisfies
  \begin{equation}
    \cR(t) \le e^{-2t}\cR(0), \quad t \ge 0,
  \end{equation}
  for any $\Phi \in \R^{d \times d}$ and $L \ge 1$.
\end{theorem}

The theorem above holds for arbitrary $\Phi$, and does not require depth or
width to be large. To prove the theorem, we first define a group of invariant
matrices as following. Note that they also play a key role in the analysis
of~\cite{arora2018convergence}.

\begin{definition}
  For a deep linear network \eqref{eqn:loss}, define the \emph{invariant matrix}
  \begin{equation}
    D_l = W_{l+1}^\intercal W_{l+1} - W_l W_l^\intercal,
    \quad l = 1, 2, \dots, L-1.
    \label{eqn:invar}
  \end{equation}
\end{definition}

\begin{lemma}\label{lem:invar}
  The invariant matrices \eqref{eqn:invar} are indeed invariances under
  continuous-time gradient descent \eqref{eqn:gd_cont}, i.e.,
  $D_l(t) = D_l(0)$ for $l = 1, \dots, L-1$ and $t \ge 0$.
\end{lemma}

\begin{proof}
  Recall that
  \[
    \dot W_l = -\nabla_l \cR
    = - W_{L:l+1}^\intercal (W_{L:1} - \Phi) W_{l-1:1}^\intercal,
  \]
  we have
  \[
    \dot W_l W_l^\intercal
    = - W_{L:l+1}^\intercal (W_{L:1} - \Phi) W_{l:1}^\intercal
    = W_{l+1}^\intercal \dot W_{l+1},
  \]
  then
  \[
    \dot D_l
    = \frac{d}{dt} \left[ W_{l+1}^T W_{l+1} - W_l W_l^\intercal \right]
    = \left[ W_{l+1}^\intercal \dot W_{l+1} - \dot W_l W_l^\intercal \right]
    + {\left[
        W_{l+1}^\intercal \dot W_{l+1} - \dot W_l W_l^\intercal
    \right]}^\intercal
    = 0.
  \]
  Therefore, $D_l(t) = D_l(0)$.
\end{proof}

\begin{proof}[Proof of Theorem~\ref{thm:cont}]
  From the ZAS initialization, $D_l(t) = D_l(0) = 0$, $l = 1, \dots,
  L-2$ and $D_{L-1}(t) = D_{L-1}(0) = -I$, i.e.,
  \begin{align*}
    W_l W_l^\intercal & = W_{l+1}^\intercal W_{l+1}, \quad l = 1, \dots, L-2, \\
    W_{L-1} W_{L-1}^\intercal & = I + W_L^\intercal W_L.
  \end{align*}
  So we have
  \begin{align*}
    W_{L-1:1} W_{L-1:1}^\intercal
    & = W_{L-1:2} W_1 W_1^\intercal W_{L-1:2}^\intercal
    = W_{L-1:2} W_2^\intercal W_2 W_{L-1:2}^\intercal\\
    & = W_{L-1:3} {(W_2 W_2^\intercal)}^2 W_{L-1:3}^\intercal\\
    & = \cdots\\
    & = {\left( W_{L-1} W_{L-1}^\intercal \right)}^{L-1}\\
    & = {\left( I + W_L^\intercal W_L \right)}^{L-1},
  \end{align*}
  and
  \begin{align}
    \|\nabla_L \cR\|_F^2
    & = \left\| (W_{L:1} - \Phi) W_{L-1:1}^\intercal \right\|_F^2 \nonumber \\
    & \ge \sigma_{\min}^2 (W_{L-1:1}) \|W_{L:1} - \Phi\|_F^2
    = \lambda_{\min} \left( W_{L-1:1} W_{L-1:1}^\intercal \right) \cdot 2 \cR
    \nonumber \\
    & = \lambda_{\min}
    \left( {\left( I + W_L^\intercal W_L \right)}^{L-1} \right)
    \cdot 2 \cR \ge 2 \cR.
    \label{eqn:grad_last}
  \end{align}
  Then
  \[
    \dot \cR(t)
    = \sum_{l=1}^L \tr\left( \nabla_l^\intercal \cR(t) \dot W_l(t) \right)
    = -\sum_{l=1}^L \|\nabla_l \cR\|_F^2
    \le -\|\nabla_L \cR\|_F^2
    \le -2 \cR.
  \]
  Therefore, $\cR(t) \le e^{-2t} \cR(0)$.
\end{proof}

\begin{remark}
  (1) For rectangular weight matrices $W_l \in \R^{d_l \times d_{l-1}}$, if $d_l
  \ge d_0 = d_\bx$, $l = 1, \dots, L-1$, we can always ignore the redundant
  nodes by initializing $W_L = 0$ and $W_l = \left[ \begin{array}{c c} I_{d_0} &
  0 \\ 0 & 0 \end{array} \right]$, then the proof of Theorem~\ref{thm:cont}
  still holds. (2) For the general square loss $\tilde\cR$ in \eqref{eqn:loss_x}
  with un-whitened data $X$, if $\lambda_X \eqdef \lambda_{\min} \left(
  X^\intercal X \right) > 0$, following the similar proof, we will have
  $\|\nabla_L \tilde\cR\|_F^2 \ge 2 \lambda_X \tilde\cR$, and $\tilde\cR(t) \le
  e^{- 2 \lambda_X t} \tilde\cR(0)$.
\end{remark}

\subsection{Discrete-time gradient descent}

Now we consider the discrete-time gradient descent \eqref{eqn:gd_discrete}. The
main theorem is stated below.

\begin{theorem}[Discrete gradient descent]\label{thm:discrete}
  For deep linear network \eqref{eqn:loss} with the zero-asymmetric
  initialization \eqref{eqn:init} and discrete-time gradient descent
  \eqref{eqn:gd_discrete}, if the learning rate satisfies 
  \[
    \eta \le \min
    \left\{ {\left( 4 L^3 \phi^6 \right)}^{-1},
    {\left( 144 L^2 \phi^4 \right)}^{-1} \right\} 
  \]
  where $\phi = \max \left\{ 2\|\Phi\|_F, 3 L^{-1/2}, 1\right\}$, then we
  have linear convergence
  \begin{equation}
    \cR(t) \le {\left( 1 - \frac{\eta}{2} \right)}^t \cR(0),
    \quad t = 0, 1, 2, \dots
    \label{eqn:converge}
  \end{equation}
\end{theorem}

Since the learning rate $\eta = O\left( L^{-3} \right)$, the theorem indicates
that gradient descent can achieve $\cR(t) \le \varepsilon$ in $O\left( L^3
\log(1/\varepsilon) \right)$ iterations.

\subsubsection{Overview of the proof}

The following is the proof sketch, and the detailed proof is deferred to the
appendix.

The approach to the discrete-time result is similar to the continuous-time case.
However, the matrices defined in \eqref{eqn:invar} are not exactly invariant,
but change slowly during the training process, which need to be controlled
carefully.

First, we propose the following three conditions, and prove that the first
condition implies the other two.

\begin{description}
  \item[Approximate invariances] For invariant matrices defined
    in \eqref{eqn:invar},
    \begin{equation}
      \|D_l\|_2 = O\left( L^{-3} \right),\ l = 1, \dots, L-2,
      \quad \text{and} \quad
      \|I + D_{L-1}\|_2 = O\left( L^{-2} \right).
      \label{eqn:invar_approx}
    \end{equation}
  \item[Weight bounds] For weight matrices $W_l$,
    \begin{equation}
      \|W_l\|_2 = 1 + O\left( \frac{\log L}{L} \right),\ l = 1, \dots, L-1,
      \quad \text{and} \quad
      \|W_{L-1}\| = O\left( L^{-1/2} \right).
      \label{eqn:weight_bound}
    \end{equation}
  \item[Gradient bound] The gradient of the last layer
    \begin{equation}
      \|\nabla_L \cR\|_F^2 \ge \cR.
      \label{eqn:grad_bound}
    \end{equation}
\end{description}

\begin{lemma}
  The approximate invariances condition \eqref{eqn:invar_approx} implies the
  weight bounds \eqref{eqn:weight_bound} and the gradient
  bound \eqref{eqn:grad_bound}.
\end{lemma}

Second, to show that \eqref{eqn:invar_approx}--\eqref{eqn:grad_bound} always
holds during the training process, we need to estimate the change of invariant
matrix $D_l(t+1) - D_l(t)$ and the decrease of loss $\cR(t+1) - \cR(t)$ in one
step.

\begin{lemma}
  If the weight bounds \eqref{eqn:weight_bound} hold at iteration $t$, then the
  change of invariant matrices after one-step update with learning rate $\eta$
  satisfies
  \begin{gather}
    \|D_l(t+1) - D_l(t)\|_2 = O\left( \eta^2 \right) \cR(t),\ l = 1, \dots, L-2,
    \nonumber \\
    \|D_{L-1}(t+1) - D_{L-1}(t)\|_2 = O\left( \eta^2 L \right) \cR(t).
    \label{eqn:invar_step}
  \end{gather}
\end{lemma}

\begin{lemma}
  If the weight bounds \eqref{eqn:weight_bound} and the gradient bound
  \eqref{eqn:grad_bound} hold, and the learning rate $\eta = O\left( L^{-2}
  \right)$, then the loss function
  \begin{equation}
    \cR(t+1) \le \left( 1 - \frac{\eta}{2} \right) \cR(t).
    \label{eqn:loss_step}
  \end{equation}
\end{lemma}

With the three lemmas above, we are now ready to prove
Theorem~\ref{thm:discrete}.

\begin{proof}[Proof of Theorem~\ref{thm:discrete} (informal)]
  We do induction on \eqref{eqn:converge} and \eqref{eqn:invar_approx}. Assume
  that they hold for $0, 1, \dots, t$. From the three lemmas above,
  \eqref{eqn:weight_bound}--\eqref{eqn:loss_step} also hold for $0, 1, \dots,
  t$. So the loss function
  \[
    \cR(t+1) \le \left( 1 - \frac{\eta}{2} \right) \cR(t)
    \le {\left( 1 - \frac{\eta}{2} \right)}^{t+1} \cR(0),
  \]
  i.e., \eqref{eqn:converge} holds for $t+1$. Now we have
  \[
    \sum_{s=0}^t \cR(s)
    \le \cR(0) \sum_{s=0}^t {\left( 1 - \frac{\eta}{2} \right)}^s
    \le \frac{2}{\eta} \cR(0) = O\left( \eta^{-1} \right).
  \]
  Recall that the invariant matrices $D_l(0) = 0$, $l=1,\dots,L-2$ and $I +
  D_{L-1}(0) = 0$ at the initialization, and $\eta = O\left( L^{-3} \right)$.
  From \eqref{eqn:invar_step},
  \[
    \|D_l(t+1)\|_2 \le \sum_{s=0}^t \|D_l(s+1) - D_l(s)\|_2
    = O(\eta^2) \sum_{s=0}^t \cR(s)
    = O(\eta) =  O\left( L^{-3} \right).
  \]
  for $l = 1, \dots, L-2$. Similarly, $\|I + D_{L-1}(t+1)\|_2 \le O(\eta L) =
  O\left( L^{-2} \right)$, i.e., \eqref{eqn:invar_approx} holds for $t+1$. Then
  we complete the induction.
\end{proof}

\begin{remark}
  Following the proof sketch, we can actually prove Theorem~\ref{thm:discrete}
  under ``near-ZAS'' initialization with perturbation: $W_l(0) \sim
  \mathcal{N}\left( I, \sigma^2 \right)$, $l = 1, \dots, L-1$ and $W_L(0) \sim
  \mathcal{N}\left( 0, \sigma^2 \right)$, where $\sigma$ is sufficiently small
  such that the approximate invariances condition \eqref{eqn:invar_approx} holds
  at the initialization. Note that the constants hidden in $O(\cdot)$ may depend
  on the target matrix $\Phi$.
\end{remark}

\section{Numerical experiments}

\subsection{Dependence on the depth}

Theorem~\ref{thm:discrete} theoretically shows that the number of iterations
required for convergence is at most $O\left( L^3 \right)$, which holds for any
target matrix in $\R^{d \times d}$. The first experiment examines how this depth
dependence behaves in practice.

In experiments, we generate target matrices in two ways:
\begin{itemize}
  \item Gaussian random matrix: $\Phi = (\phi_{ij}) \in \R^{d \times d}$ with
    $\phi_{ij}$ independently drawn from $\cN(0, 1)$. Both $d = 2$ and $d =
    100$ are considered.
  \item Negative identity matrix: $\Phi = -I \in \R^{d \times d}$. This target
    is adopted from \cite{shamir2018exponential}, which proves that in the case
    $d = 1$, the number of iteration required for convergence under the Xavier
    and the near-identity initialization scales exponentially with the depth
    $L$. Both $d = 1$ and $d = 100$ are considered.
\end{itemize} 

The ZAS initialization \eqref{eqn:init} is applied for linear neural networks
with different depth $L$, and we manually tune the optimal learning rate for
each $L$. As suggested by Theorem~\ref{thm:discrete}, we numerically find that
the optimal learning rate decrease with $L$. 

Figure~\ref{fig:converge1} shows number of iterations required to make the
objective $\cR \le \varepsilon = 10^{-10}$. It is clear to see that the number
of iterations required roughly scale as $O\left( L^\gamma \right)$, where
$\gamma \approx 1/2$ for the negative identity matrix and $\gamma \approx 1$ for
the Gaussian random matrices. These scalings are better than the theoretical
$\gamma = 3$ in Theorem~\ref{thm:discrete}, which is a worst case result.

\begin{figure}[t]
  \centering 
  \includegraphics[width=0.4\textwidth]{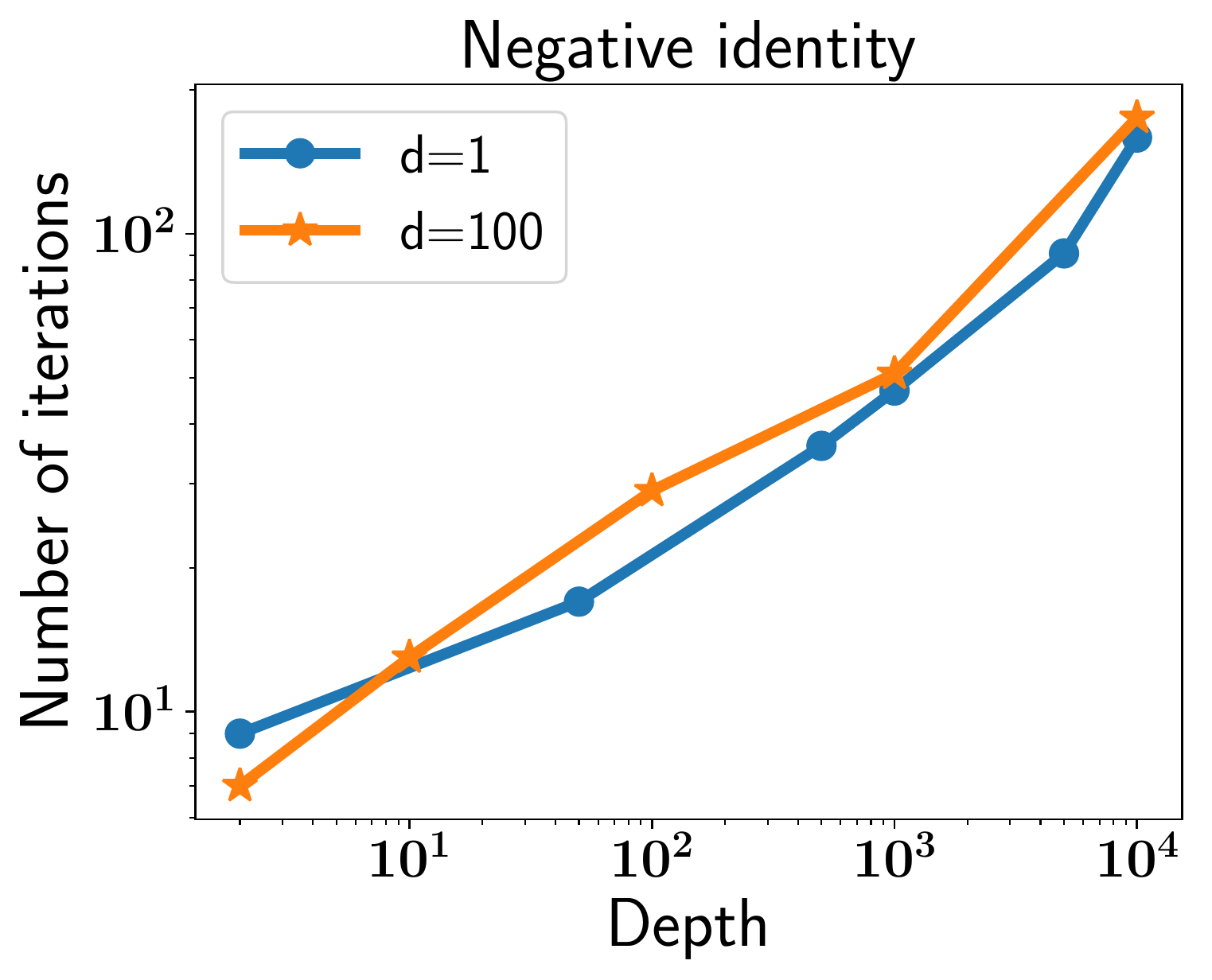}
  \qquad
  \includegraphics[width=0.4\textwidth]{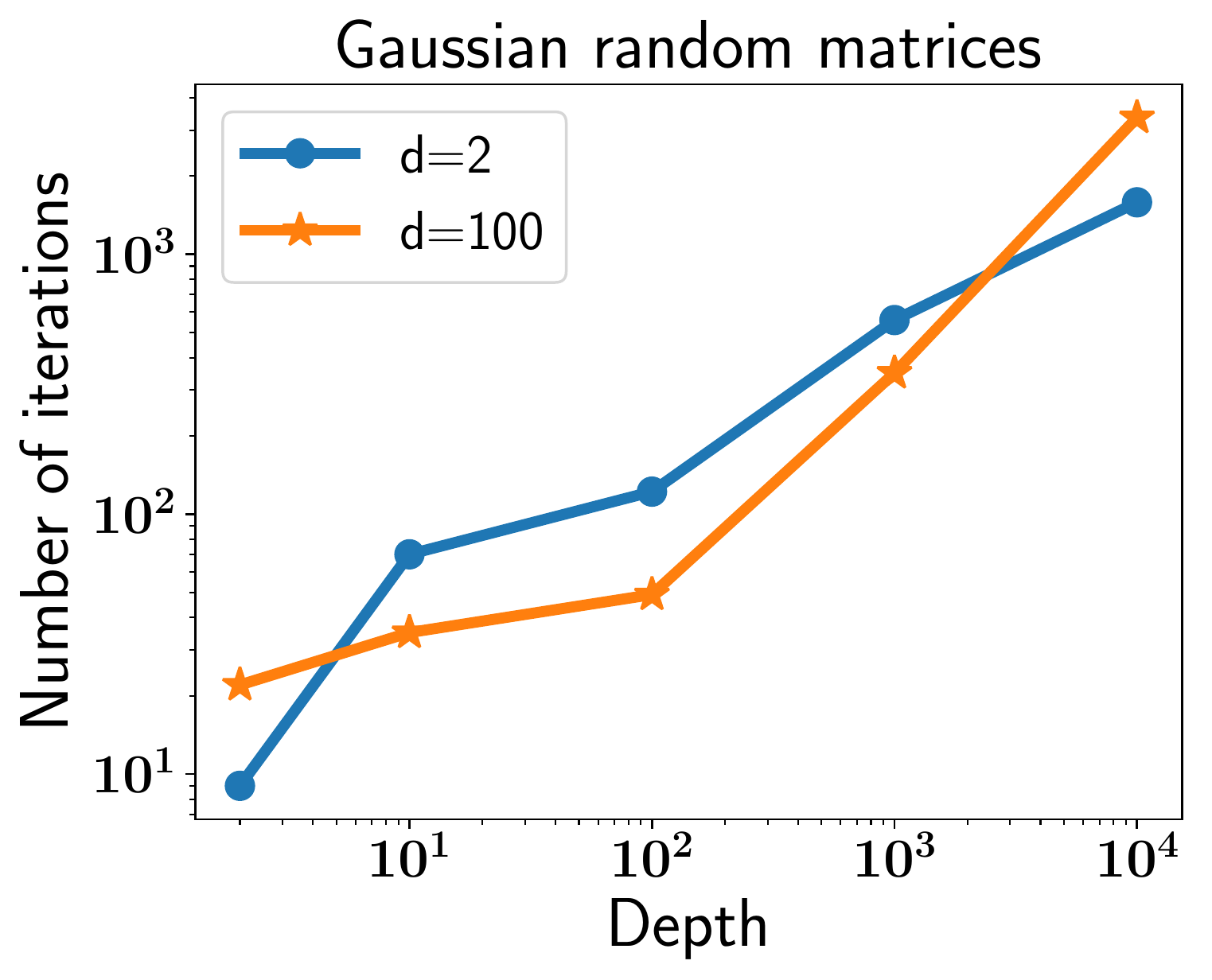}
  \caption{%
    Number of iterations required for the ZAS initialization to reach an
    $\varepsilon$-optimal solution where $\varepsilon = 10^{-10}$. Two type of
    target matrices, negative identity and Gaussian random matrices are
    considered. It is shown that the number of iterations required scales
    polynomially with the network depth.
  }%
  \label{fig:converge1}
\end{figure}

\subsection{Comparison with near-identity initialization in multi-dimensional
cases}

The near-identity initialization initializes each layer by 
\begin{equation}\label{eqn:near-id}
  W_l = I + U_l, \quad {(U_l)}_{ij} \sim \cN(0, 1 / (d L))\ \textrm{i.i.d.},
  \quad l = 1, \dots, L
\end{equation}
where $I$ is the identity matrix. Numerically, it was observed in
\cite{shamir2018exponential} that for multi-dimensional networks ($d = 25$ in
the experiments), gradient descent with the initialization \eqref{eqn:near-id}
requires number of iterations to scale only polynomially with the depth, instead
of exponentially. Here we compare it with the ZAS initialization by fitting
negative identity matrix with 6-layer linear networks. The learning rate $\eta =
0.01$ for both initialization.

Figure~\ref{fig:comparison} shows the dynamics trajectories for both
initializations. It strongly suggests that the ZAS initialization is more
efficient than the near-identity initialization \eqref{eqn:near-id}. Gradient
descent with the near-identity initialization is attracted to a saddle region,
spends a long time escaping that region, and then converges fast to a global
minimum. As a comparison, gradient descent with ZAS initialization does not
encounter any saddle region during the whole optimization process.

\begin{figure}[t]
  \centering
  \includegraphics[width=0.4\textwidth]{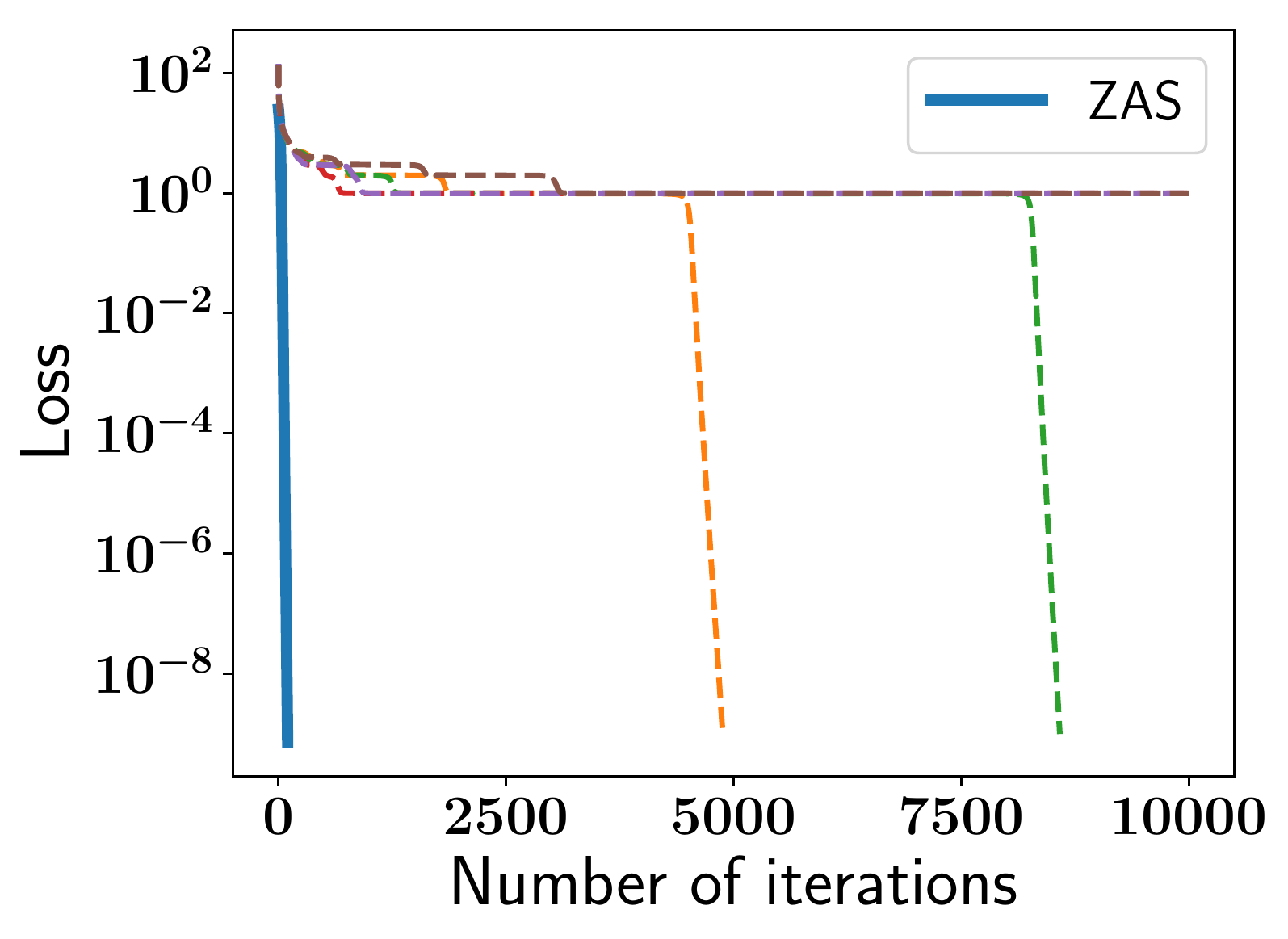}
  \qquad
  \includegraphics[width=0.4\textwidth]{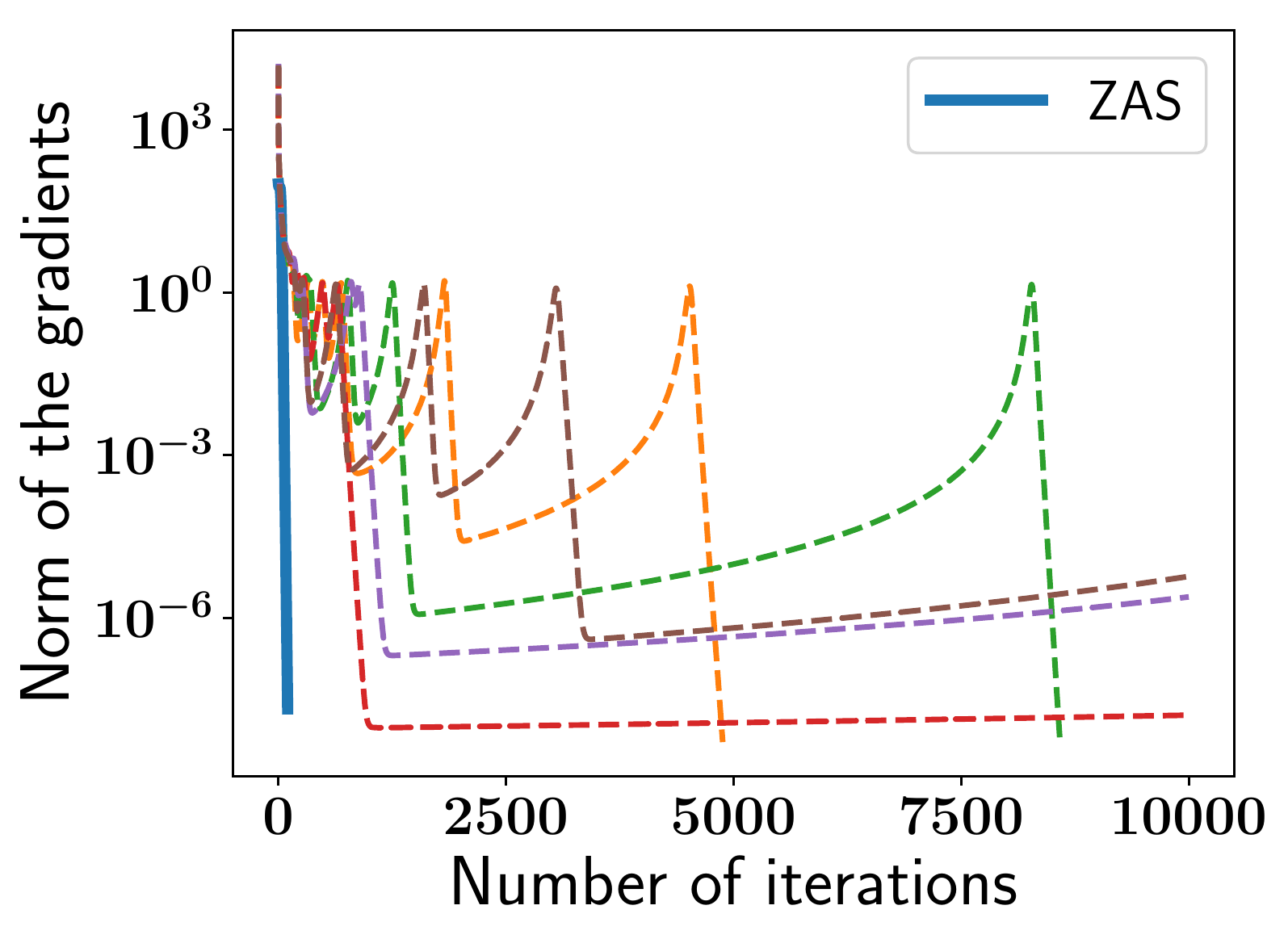}
  \caption{%
    Comparison between the ZAS and the near-identity initialization. The 5
    dashed lines correspond to the multiple runs of gradient descent with the
    near-identity initialization. It is shown that GD with the near-identity
    successfully escape the saddle region only 2 of 5 times in the given number
    of iterations, while the ZAS does not suffer from the attraction of saddle
    point at all.
  }%
  \label{fig:comparison}
\end{figure}

\section{An extension to nonlinear residual networks}

Consider the following residual network $f: \R^d \to \R^{d'}$:
\begin{align}
  \bz_0 & = V_0 \bx, \nonumber \\ 
  \bz_l & = \bz_{l-1} + U_l \sigma(V_l \bz_{l-1}), \quad l = 1, \dots, L,
  \nonumber \\
  f(\bx) & = U_{L+1} \bz_L,
  \label{eqn:resnet}
\end{align}
where $V_0 \in \R^{D \times d}$, $U_l \in \R^{D \times m}$, $V_l \in \R^{m
\times D}$ and $U_{L+1} \in \R^{d' \times D}$; $d$ is the input dimension, $d'$
is the output dimension, $m$ is the width of the residual blocks and $D$ is the
width of skip connections. 

For the nonlinear residual network \eqref{eqn:resnet}, we propose the following
\emph{modified ZAS (mZAS) initialization}:
\begin{gather}
  U_l = 0, \quad l = 1, 2, \dots, L+1, \nonumber \\
  {(V_l)}_{ij} \sim \cN(0, 1/D)\ \textrm{i.i.d.,} \quad l = 0, 1, \dots, L. 
  \label{eqn:mZAS}
\end{gather}

We test two types of initialization: (1) standard Xavier initialization; (2)
mZAS initialization \eqref{eqn:mZAS}. The experiments are conducted on
Fashion-MNIST \cite{xiao2017fashion}, where we select 1000 training samples
forming the new training set to speed up the computation. Depth $L=100, 200,
2000, 10000$ are tested, and the learning rate for each depth is tuned to the
achieve the fastest convergence. The results are displayed in
Figure~\ref{fig:nonlinear}. 

It is shown that mZAS initialization always outperforms the Xavier
initialization. Moreover, gradient descent with mZAS initialization is even able
to successfully optimize a 10000-layer residual network. It clearly demonstrates
that the ZAS-type initialization can be helpful for optimizing deep nonlinear
residual networks.

\begin{figure}[t]
  \centering
  \includegraphics[width=.4\textwidth]{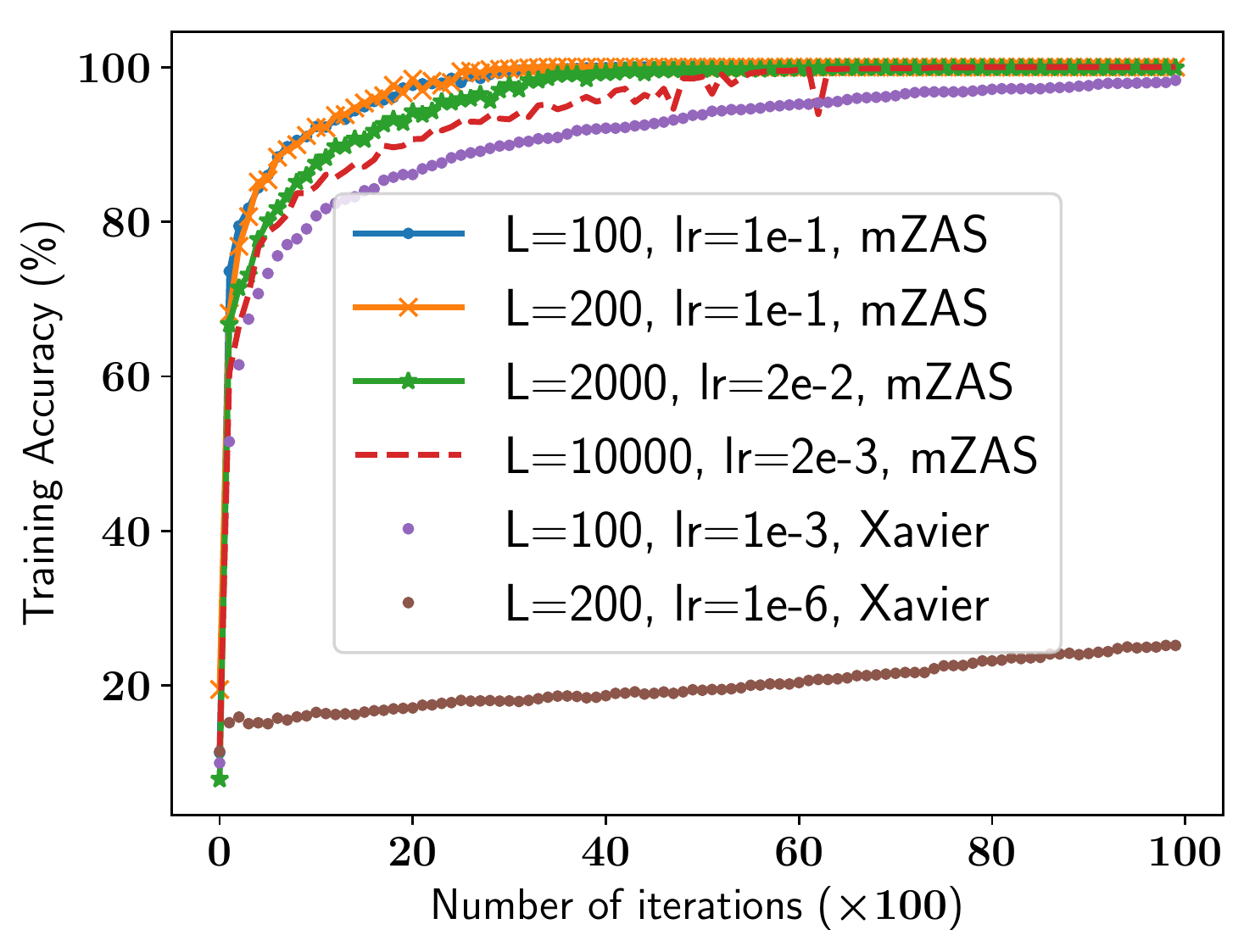}
  \caption{%
    For the nonlinear residual network and Fashion-MNIST dataset, the mZAS
    initialization outperforms the Xavier initialization. The latter blows up for
    depth $L = 2000, 10000$. The learning rates are tuned to achieve the fastest
    convergence.
  }%
  \label{fig:nonlinear}
\end{figure}

\section{Conclusion}

In this paper we propose the ZAS initialization for deep linear residual
network, under which gradient descent converges to global minima for arbitrary
target matrices with linear rate.  Moreover, the rate only scales polynomially
with the network depth. Numerical experiments show that the ZAS initialization
indeed avoids the attraction of saddle points, comparing to the near-identity
initialization. This type of initialization may be extended to the analysis of
deep nonlinear residual networks, which we leave as future work.

\subsubsection*{Acknowledgments}
We are grateful to Prof.~Weinan E for helpful discussions, and the anonymous
reviewers for valuable comments and suggestions. This work is supported in part
by a gift to Princeton University from iFlytek and the ONR grant
N00014-13-1-0338.

\bibliographystyle{plain}
\bibliography{dl_ref}

\newpage
\appendix

\section{Proof of the discrete-time gradient descent}%
\label{sec:discrete}

\subsection{Invariant matrices}

The zero-asymmetry initialization \eqref{eqn:init} gives $D_{l}(0)=0$,
$l=1,\dots,L-2$ and $I+D_{L-1}(0)=0$. Lemma~\ref{lem:invar} proved that
$D_{l}$'s are indeed invariances in continuous gradient descent, and then the
gradient $\|\nabla_{L}\cR\|_{F}^{2}$ can be lower bounded by the current loss
$\cR$ \eqref{eqn:grad_last}. Here we will show that
$\|\nabla_{L}\cR\|_{F}^{2}\ge\cR$ still holds if $D_{l}$'s are only
\emph{approximately invariant}, i.e., $D_{l}$, $l=1,\dots,L-2$ and $I+D_{L-1}$
are close to 0.

\begin{lemma}\label{lem:invar_mult}
Assume that the weight matrices $\|W_{l}\|_{2}\le\alpha$,
$l=1,\dots,L-1$ and the invariant matrices $\|D_{l}\|_{2}\le\delta$,
$l=1,\dots,L-2$, then
\begin{equation}
\left\|W_{L-1:1}W_{L-1:1}^{\intercal}-{\left(W_{L-1}W_{L-1}^{\intercal}\right)}^{L-1}\right\|_{2}\le\frac{1}{2}L^{2}\alpha^{2(L-2)}\delta.
\end{equation}
\end{lemma}

\begin{proof}
We will proof the following statement by induction
\begin{equation}
\left\|W_{l:1}W_{l:1}^{\intercal}-{\left(W_{l}W_{l}^{\intercal}\right)}^{l}\right\|_{2}\le\frac{l(l-1)}{2}\alpha^{2(l-1)}\delta,\quad l=1,\dots,L-1.
\label{eqn:induc_invar_mult}
\end{equation}

The statement holds for $l=1$ obviously. Assume that the statement
holds for $l$, now consider $l+1$,
\begin{align*}
 & \left\|W_{l+1:1} W_{l+1:1}^{\intercal}-{\left(W_{l+1}W_{l+1}^{\intercal}\right)}^{l+1}\right\|_{2}\\
 & =\left\|W_{l+1}\left[W_{l:1}W_{l:1}^{\intercal}-{\left(W_{l}W_{l}^{\intercal}\right)}^{l}+{\left(W_{l}W_{l}^{\intercal}\right)}^{l}-{\left(W_{l+1}^{\intercal}W_{l+1}\right)}^{l}\right]W_{l+1}^{\intercal}\right\|_{2}\\
 & \le\|W_{l+1}\|_{2}\left[\left\|W_{l:1}W_{l:1}^{\intercal}-{\left(W_{l}W_{l}^{\intercal}\right)}^{l}\right\|_{2}+\left\|{\left(W_{l}W_{l}^{\intercal}\right)}^{l}-{\left(W_{l+1}^{\intercal}W_{l+1}\right)}^{l}\right\|_{2}\right]\|W_{l+1}^{\intercal}\|_{2}\\
 & \le\alpha^{2}\left[\frac{l(l-1)}{2}\alpha^{2(l-1)}\delta+\left\|{\left(W_{l}W_{l}^{\intercal}\right)}^{l}-{\left(W_{l+1}^{\intercal}W_{l+1}\right)}^{l}\right\|_{2}\right],
\end{align*}
and
\begin{align*}
 & \left\|{\left(W_{l}W_{l}^{\intercal}\right)}^{l}-{\left(W_{l+1}^{\intercal}W_{l+1}\right)}^{l}\right\|_{2}\\
 & =\left\|\sum_{k=0}^{l-1}{\left(W_{l}W_{l}^{\intercal}\right)}^{l-1-k}\left(W_{l}W_{l}^{\intercal}-W_{l+1}^{\intercal}W_{l+1}\right){\left(W_{l+1}^{\intercal}W_{l+1}\right)}^{k}\right\|_{2}\\
 & \le\sum_{k=0}^{l-1}\left\|W_{l}W_{l}^{\intercal}\right\|_{2}^{l-1-k}\left\|W_{l}W_{l}^{\intercal}-W_{l+1}^{\intercal}W_{l+1}\right\|_{2}\left\|W_{l+1}^{\intercal}W_{l+1}\right\|_{2}^{k}\\
 & \le\sum_{k=0}^{l-1}\alpha^{2(l-1-k)}\delta\alpha^{2k}=l\alpha^{2(l-1)}\delta,
\end{align*}
thus
\[
\left\|W_{l+1:1}W_{l+1:1}^{\intercal}-{\left(W_{l+1}W_{l+1}^{\intercal}\right)}^{l+1}\right\|_{2}\le\alpha^{2}\left[\frac{l(l-1)}{2}\alpha^{2(l-1)}\delta+l\alpha^{2(l-1)}\delta\right]=\frac{l(l+1)}{2}\alpha^{2l}\delta.
\]
So the statement \eqref{eqn:induc_invar_mult} also holds for $l+1$,
and we complete the proof of the lemma.
\end{proof}

\begin{lemma}\label{lem:grad_last}
Assume that the weight matrices $\|W_{l}\|_{2}\le\alpha$,
$l=1,\dots,L-1$, where $1\le\alpha^{2(L-1)}<L\phi^{2}$ for some
$\phi>0$; assume that the invariant matrices $\|D_{l}\|_{2}\le\delta$,
$l=1,\dots,L-2$ and $\|I+D_{L-1}\|_{2}\le\varepsilon$, where $\delta\le{\left(2L^{3}\phi^{2}\right)}^{-1}$
and $\varepsilon\le{\left(4L^2\right)}^{-1}$. Then $\|\nabla_L \cR\|_{F}^{2}\ge\cR$.
\end{lemma}

\begin{proof}
From Lemma~\ref{lem:invar_mult},
\[
\lambda_{\min}\left(W_{L-1:1}W_{L-1:1}^{\intercal}\right)\ge\lambda_{\min}^{L-1}\left(W_{L-1}W_{L-1}^{\intercal}\right)-\frac{1}{2}L^{2}\alpha^{2(L-2)}\delta.
\]
Since $\left\|I+W_{L}^{\intercal}W_{L}-W_{L-1}W_{L-1}^{\intercal}\right\|_{2}\le\varepsilon$,
\[
\lambda_{\min}(W_{L-1}W_{L-1}^{\intercal})\ge\lambda_{\min}\left(I+W_{L}^{\intercal}W_{L}\right)-\varepsilon\ge1-\varepsilon.
\]
Similar to \eqref{eqn:grad_last}, we have
\begin{multline*}
\|\nabla_{L}\cR\|_{F}^{2}\ge2\lambda_{\min}\left(W_{L-1:1}W_{L-1:1}^{\intercal}\right)\cR\\
\ge2\left[{(1-\varepsilon)}^{L-1}-\frac{1}{2}L^{2}\alpha^{2(L-2)}\delta\right]\cR\ge2\left[1-\frac{L-1}{4L^{2}}-\frac{1}{2}L^{2}\cdot L\phi^{2}\cdot\frac{1}{2L^{3}\phi^{2}}\right]\cR\ge\cR.
\end{multline*}
\end{proof}

In addition, if $D_{l}$'s are approximately invariant, we can bound
the weights $\|W_{l}\|_{2}$.

\begin{lemma}\label{lem:weight_bound}
Let $\alpha=\max_{1\le l\le L-1}\|W_{l}\|_{2}\vee1$,
$\beta=\|W_{L}\|_{2}$ and $\phi=\max\left\{ \|W_{L:1}\|_{2},\frac{e}{\sqrt{L}},1\right\} $.
Assume that the invariant matrices $\|D_{l}\|_{2}\le\delta$, $l=1,\dots,L-2$
and $\|I+D_{L-1}\|_{2}\le\varepsilon$, where $\delta\le{\left(2L^{3}\phi^{2}\right)}^{-1}$
and $\varepsilon\le{\left(4L^2\right)}^{-1}$. Then
\begin{equation}
\alpha^{2(L-1)}<L\phi^{2},\quad\alpha^{2(L-1)}\beta^{2}<2\phi^{2}.
\end{equation}
\end{lemma}

\begin{proof}
We first use the invariant matrices to bound the difference between
$\|W_{l}\|_{2}$. Since
\begin{multline*}
\left\|I+D_{L-1}\right\|_{2}=\left\|I+W_{L}^{\intercal}W_{L}-W_{L-1}W_{L-1}^{\intercal}\right\|_{2}\\
\ge\left|\left\|I+W_{L}^{\intercal}W_{L}\right\|_{2}-\left\|W_{L-1}W_{L-1}^{\intercal}\right\|_{2}\right|=\left|1+\|W_{L}\|_{2}^{2}-\|W_{L-1}\|_{2}^{2}\right|,
\end{multline*}
we have $\left|1+\beta^{2}-\|W_{L-1}\|_{2}^{2}\right|\le\varepsilon$.
In addition,
\[
\|D_{l}\|_{2}=\left\|W_{l+1}^{\intercal}W_{l+1}-W_{l}W_{l}^{\intercal}\right\|_{2}\ge\left|\left\|W_{l+1}^{\intercal}W_{l+1}\right\|_{2}-\left\|W_{l}W_{l}^{\intercal}\right\|_{2}\right|=\left|\|W_{l+1}\|_{2}^{2}-\|W_{l}\|_{2}^{2}\right|
\]
for $l=1,\dots,L-2$, then $\left|1+\beta^{2}-\|W_{l}\|_{2}^{2}\right|\le\varepsilon+(L-l-1)\delta$,
thus $\left|1+\beta^{2}-\alpha^{2}\right|\le\varepsilon+(L-2)\delta$.

From Lemma~\ref{lem:invar_mult},
\begin{align*}
W_{L:1}W_{L:1}^{\intercal} & =W_{L}\left[W_{L-1:1}W_{L-1:1}^{\intercal}\right]W_{L}^{\intercal}\\
 & \succeq W_{L}\left[{\left(W_{L-1}W_{L-1}^{\intercal}\right)}^{L-1}-\frac{1}{2}\alpha^{2(L-2)}L^{2}\delta I\right]W_{L}^{\intercal}\\
 & \succeq W_{L}\left[{\left(I+W_{L}^{\intercal}W_{L}-\delta I\right)}^{L-1}-\frac{1}{2}\alpha^{2(L-2)}L^{2}\delta I\right]W_{L}^{\intercal},
\end{align*}
where $A\succeq B$ means the matrix $A-B$ is positive semi-definite. So
\begin{align*}
\left\|W_{L:1}W_{L:1}^{\intercal}\right\|_{2} & \ge\left\|W_{L}\left[{\left(I+W_{L}^{\intercal}W_{L}-\delta I\right)}^{L-1}-\frac{1}{2}\alpha^{2(L-2)}L^{2}\delta I\right]W_{L}^{\intercal}\right\|_{2}\\
 & =\beta^{2}\left[{\left(1+\beta^{2}-\varepsilon\right)}^{L-1}-\frac{1}{2}\alpha^{2(L-2)}L^{2}\delta\right]\\
 & \ge\beta^{2}\left[{\left(\alpha^{2}-2\varepsilon-(L-2)\delta\right)}^{L-1}-\frac{1}{2}\alpha^{2(L-2)}L^{2}\delta\right]\\
 & \ge\beta^{2}\left[\alpha^{2(L-1)}-(L-1)(2\varepsilon+(L-2)\delta)-\frac{1}{2}\alpha^{2(L-2)}L^{2}\delta\right],\\
 & \ge\beta^{2}\left[\alpha^{2(L-1)}-\frac{{(L-1)}^{2}}{L^{3}}-\frac{1}{4L}\alpha^{2(L-2)}\right]\\
 & \ge\frac{1}{2}\alpha^{2(L-1)}\beta^{2}
\end{align*}
since $\delta\le{\left(2L^{3}\right)}^{-1}$ and $\varepsilon\le{\left(4L^{2}\right)}^{-1}$.
Therefore, $\alpha^{2(L-1)}\beta^{2}\le2\left\|W_{L:1}W_{L:1}^{\intercal}\right\|_{2}\le2\phi^{2}$.

Finally, assume that $\alpha^{2(L-1)}\ge L\phi^{2}$, then
\begin{gather*}
\alpha^{2}\ge{\left(L\phi^{2}\right)}^{1/(L-1)}=\exp\left[\frac{\log\left(L\phi^{2}\right)}{L-1}\right]>1+\frac{\log\left(L\phi^{2}\right)}{L-1},\\
\beta^{2}\ge\frac{\log\left(L\phi^{2}\right)}{L-1}-\varepsilon-(L-2)\delta\ge\frac{2}{L-1}-\frac{1}{4L^{2}}-\frac{L-2}{2L^{3}}>\frac{2}{L},
\end{gather*}
where $\log(L\phi^{2})\ge2$ comes from $\phi\ge \frac{e}{\sqrt L}$. Thus
\[
\left\|W_{L:1}W_{L:1}^{\intercal}\right\|_{2}\ge\frac{1}{2}\alpha^{2(L-1)}\beta^{2}>\frac{1}{2}\cdot L\phi^{2}\cdot\frac{2}{L}=\phi^{2},
\]
which is a contradiction! Therefore $\alpha^{2(L-1)}<L\phi^{2}$,
and we complete the proof of the lemma.
\end{proof}

\subsection{One-step analysis}

We denote the one-step update as
\[
W_{l}^{+}=W_{l}-\eta\nabla_{l}\cR,\quad l=1,\dots,L.
\]
In this section, we always denote $A^{+}$ as the value of a variable
$A$ after one-step update, for example $\cR^{+}$, $W_{l_{2}:l_{1}}^{+}$
and $D_{l}^{+}$. We will estimate the change of invariant matrix
$D_{l}^{+}-D_{l}$ and the change of loss $\cR^{+}-\cR$ in one step.

\begin{lemma}\label{lem:grad_bound}
Assume that $\|W_{l}\|_{2}\le\alpha$, $l=1,\dots,L-1$
and $\|W_{L}\|_{2}\le\beta$, where $1\le\alpha^{2(L-1)}<L\phi^{2}$
and $\alpha^{2(L-1)}\beta^{2}<2\phi^{2}$ for some $\phi>0$. Then
\begin{gather*}
\|\nabla_{l}\cR\|_{F}^{2}\le4\phi^{2}\cR,\quad l=1,\dots,L-1,\\
\|\nabla_{L}\cR\|_{F}^{2}\le2L\phi^{2}\cR.
\end{gather*}
\end{lemma}

\begin{proof}
For $l=1,\dots,L-1$,
\begin{multline*}
\left\|\nabla_{l}\cR\right\|_{F}=\left\|W_{L:(l+1)}^{\intercal}(W_{L:1}-\Phi)W_{(l-1):1}^{\intercal}\right\|_{F}\le\left\|W_{L:(l+1)}\right\|_{2}\left\|W_{L:1}-\Phi\right\|_{F}\left\|W_{(l-1):1}\right\|_{2}\\
\le\alpha^{L-2}\beta\sqrt{2\cR}\le2\phi\sqrt{\cR}.
\end{multline*}
And similarly, $\|\nabla_{L}\cR\|_{F}\le\alpha^{L-1}\sqrt{2\cR}\le\phi\sqrt{2L\cR}$.
\end{proof}

\begin{lemma}\label{lem:invar_step}
Under the same conditions as Lemma~\ref{lem:grad_bound},
the change of invariant matrices under one-step update satisfies
\begin{gather*}
\|D_{l}^{+}-D_{l}\|_{2}\le8\eta^{2}\phi^{2}\cR,\quad l=1,\dots,L-2,\\
\|D_{L-1}^{+}-D_{L-1}\|_{2}\le2\eta^{2}(L+2)\phi^{2}\cR.
\end{gather*}
\end{lemma}

\begin{proof}
Recall the invariance condition
\[
\nabla_{l}\cR W_{l}^{\intercal}=W_{L:(l+1)}^{\intercal}\left(W_{L:1}-\Phi\right)W_{l:1}^{\intercal}=W_{l+1}^{\intercal}\nabla_{l+1}\cR,
\]
we have
\begin{align*}
D_{l}^{+} & ={(W_{l+1}^{+})}^{\intercal}W_{l+1}^{+}-W_{l}{(W_{l}^{+})}^{\intercal}\\
 & ={\left(W_{l+1}-\eta\nabla_{l+1}\cR\right)}^{\intercal}\left(W_{l+1}-\eta\nabla_{l+1}\cR\right)-\left(W_{l}-\eta\nabla_{l}\cR\right){\left(W_{l}-\eta\nabla_{l}\cR\right)}^{\intercal}\\
 & =W_{l+1}^{\intercal}W_{l+1}-W_{l}W_{l}^{\intercal}\\
 & \quad-\eta\left[W_{l+1}^{\intercal}\nabla_{l+1}\cR-\nabla_{l}\cR W_{l}^{\intercal}+\nabla_{l+1}^{\intercal}\cR W_{l+1}-W_{l}\nabla_{l}^{\intercal}\cR\right]\\
 & \quad+\eta^{2}\left[\nabla_{l+1}^{\intercal}\cR\nabla_{l+1}\cR+\nabla_{l}\cR\nabla_{l}^{\intercal}\cR\right]\\
 & =D_{l}+\eta^{2}\left[\nabla_{l+1}^{\intercal}\cR\nabla_{l+1}\cR+\nabla_{l}\cR\nabla_{l}^{\intercal}\cR\right].
\end{align*}
Combining with Lemma~\ref{lem:grad_bound}, we can complete the proof.
\end{proof}

\begin{lemma}\label{lem:loss_step}
Under the same conditions as Lemma~\ref{lem:grad_last},
for learning rate
\[
\eta\le\min\left\{ \frac{1}{64L^{2}\phi^{3}\sqrt{\cR}},\frac{1}{144L^{2}\phi^{4}}\right\} ,
\]
the decrease of the loss function in one-step update satisfies
\[
\cR^{+}\le\left(1-\frac{\eta}{2}\right)\cR.
\]
\end{lemma}

\begin{proof}
First we expand $W_{L:1}^{+}$ as a polynomials of $\eta$:
\[
W_{L:1}^{+}=\prod_{l=1}^{L}\left(W_{l}-\eta\nabla_{l}\cR\right)=A_{0}+\eta A_{1}+\eta^{2}A_{2}+\cdots+\eta^{L}A_{L},
\]
where the coefficients $A_{k}\in\R^{d\times d}$. Obviously $A_{0}=W_{L:1}$.
\begin{align*}
\cR^{+}-\cR & =\frac{1}{2}\left[\left\|W_{L:1}^{+}-\Phi\right\|_{F}^{2}-\left\|W_{L:1}-\Phi\right\|_{F}^{2}\right]\\
 & =\frac{1}{2}\left(W_{L:1}^{+}-W_{L:1}\right):\left(W_{L:1}^{+}+W_{L:1}-2\Phi\right)\\
 & =\frac{1}{2}\left(W_{L:1}^{+}-W_{L:1}\right):\left(2\left(W_{L:1}-\Phi\right)+\left(W_{L:1}^{+}-W_{L:1}\right)\right)\\
 & =\left(W_{L:1}^{+}-W_{L:1}\right):\left(W_{L:1}-\Phi\right)+\frac{1}{2}\left\|W_{L:1}^{+}-W_{L:1}\right\|_{F}^{2},
\end{align*}
where $A:B=\sum_{i,j}A_{ij}B_{ij}$. We can write
\[
\cR^{+}-\cR=I_{1}+I_{2}+I_{3},
\]
where
\[
I_{1}=\eta A_{1}:\left(W_{L:1}-\Phi\right),\quad I_{2}=\sum_{k=2}^{L}\eta^{k}A_{k}:\left(W_{L:1}-\Phi\right),\quad I_{3}=\frac{1}{2}\left\|\sum_{k=1}^{L}\eta^{k}A_{k}\right\|_{F}^{2}.
\]

For $I_{1}$, we have
\begin{multline*}
I_{1}=A_{1}:\left(W_{L:1}-\Phi\right)=-\eta\sum_{l=1}^{L}\left(W_{L:l+1}\nabla_{l}\cR W_{l-1:1}\right):\left(W_{L:1}-\Phi\right)\\
=-\eta\sum_{l=1}^{L}\nabla_{l}\cR:\left[W_{L:l+1}^{\intercal}\left(W_{L:1}-\Phi\right)W_{l-1:1}^{\intercal}\right]=-\eta\sum_{l=1}^{L}\left\|\nabla_{l}\cR\right\|_{F}^{2}.
\end{multline*}
From Lemma~\ref{lem:grad_last},
\[
I_{1}\le-\eta\|\nabla_L \cR\|_{F}^{2}\le-\eta\cR.
\]

For $I_{2}$ and $I_{3}$, we further expand $W_{L-1:1}^{+}$ as
\[
W_{L-1:1}^{+}=\prod_{l=1}^{L-1}\left(W_{l}-\eta\nabla_{l}\cR\right)=B_{0}+\eta B_{1}+\eta^{2}B_{2}+\cdots+\eta^{L-1}B_{L-1}.
\]
From Lemma~\ref{lem:grad_bound}, $\|\nabla_{l}\cR\|_{F}\le\gamma=2\phi\sqrt{\cR}$,
$l=1,\dots,L-1$, then for $k\ge1$,
\[
\|B_{k}\|_{F}\le\binom{L-1}{k}\alpha^{L-1-k}{\left(2\phi\sqrt{\cR}\right)}^{k}.
\]
We use the following inequalities for $0\le y\le x/L^{2}$:
\[
{(x+y)}^{L}\le2x^{L},\quad{(x+y)}^{L}\le x^{L}+2Lx^{L-1}y,\quad{(x+y)}^{L}\le x^{L}+Lx^{L-1}y+L^{2}x^{L-2}y^{2}.
\]
Since $2\eta\phi\sqrt{\cR}\le\alpha/L^{2}$,
\begin{align*}
\left\|\sum_{k=0}^{L-1}\eta^{k}B_{k}\right\|_{2} & \le{\left(\alpha+2\eta\phi\sqrt{\cR}\right)}^{L-1}\le2\alpha^{L-1},\\
\left\|\sum_{k=1}^{L-1}\eta^{k}B_{k}\right\|_{F} & \le{\left(\alpha+2\eta\phi\sqrt{\cR}\right)}^{L-1}-\alpha^{L-1}\le2L\alpha^{L-2}\cdot2\eta\phi\sqrt{\cR}=4\eta L\alpha^{L-2}\phi\sqrt{\cR},\\
\left\|\sum_{k=2}^{L-1}\eta^{k}B_{k}\right\|_{F} & \le{\left(\alpha+2\eta\phi\sqrt{\cR}\right)}^{L-1}-\alpha^{L-1}-(L-1)\alpha^{L-2}\cdot2\eta\phi\sqrt{\cR}\\
 & \le L^{2}\alpha^{L-3}{\left(2\eta\phi\sqrt{\cR}\right)}^{2}=4\eta^{2}L^{2}\alpha^{L-3}\phi^{2}\cR.
\end{align*}
Notice that $A_{k}=W_{L}B_{k}-\nabla_{L}\cR B_{k-1}$, $k=1,\dots,L$
where $\|W_{L}\|_{2}\le\beta$ and $\|\nabla_{L}\cR\|_{F}\le\alpha^{L-1}\sqrt{2\cR}$,
then
\begin{align*}
\left\|\sum_{k=1}^{L}\eta^{k}A_{k}\right\|_{F} & \le\|W_{L}\|_{2}\left\|\sum_{k=1}^{L}\eta^{k}B_{k}\right\|_{F}+\eta\|\nabla_L \cR\|_{F}\left\|\sum_{k=0}^{L}\eta^{k}B_{k}\right\|_{2}\\
 & \le\beta\cdot4\eta L\alpha^{L-2}\phi\sqrt{\cR}+\eta\alpha^{L-1}\sqrt{2\cR}\cdot2\alpha^{L-1}\\
 & \le4\eta L\phi^{2}\sqrt{2\cR}+2\eta L\alpha^{2}\phi^{2}\sqrt{2\cR}\\
 & =6\eta L\phi^{2}\sqrt{2\cR},
\end{align*}
\begin{align*}
\left\|\sum_{k=2}^{L}\eta^{k}A_{k}\right\|_{F} & \le\|W_{L}\|_{2}\left\|\sum_{k=2}^{L}\eta^{k}B_{k}\right\|_{F}+\eta\|\nabla_L \cR\|_{F}\left\|\sum_{k=1}^{L}\eta^{k}B_{k}\right\|_{F}\\
 & \le\beta\cdot4\eta^{2}L^{2}\alpha^{L-3}\phi^{2}\cR+\eta\alpha^{L-1}\sqrt{2\cR}\cdot4\eta L\alpha^{L-2}\phi\sqrt{\cR}\\
 & \le4\sqrt{2}\eta^{2}L^{2}\phi^{3}\cR+4\sqrt{2}\eta^{2}L^{2}\phi^{3}\cR\\
 & =8\sqrt{2}\eta^{2}L^{2}\phi^{3}\cR.
\end{align*}
So
\begin{gather*}
I_{2}\le\left\|\sum_{k=2}^{L}\eta^{k}A_{k}\right\|_{F}\left\|W_{L:1}(k)-\Phi\right\|_{F}\le16\eta^{2}L^{2}\phi^{3}\cR^{3/2},\\
I_{3}=\frac{1}{2}\left\|\sum_{k=1}^{L}\eta^{k}A_{k}\right\|_{F}^{2}\le36\eta^{2}L^{2}\phi^{4}\cR.
\end{gather*}
For $\eta\le\min\left\{ {\left(64L^{2}\phi^{3}\sqrt{\cR}\right)}^{-1},{\left(144L^{2}\phi^{4}\right)}^{-1}\right\} $,
we have $I_{2}\le\eta\cR/4$ and $I_{3}\le\eta\cR/4$. Therefore,
\[
\cR^{+}-\cR=I_{1}+I_{2}+I_{3}\le-\eta\cR+\frac{1}{4}\eta\cR+\frac{1}{4}\eta\cR=-\frac{1}{2}\eta\cR.
\]
\end{proof}

\subsection{Proof of Theorem~\ref{thm:discrete}}

Now we are ready to prove the main Theorem~\ref{thm:discrete}.

\begin{proof}[Proof of Theorem~\ref{thm:discrete}]
Let $\alpha(t)=\max_{1\le l\le L-1}\|W_{l}(t)\|_{2}\vee1$,
and $\beta(t)=\|W_{L}(t)\|_{2}$. We will proof the following two
statements by induction:
\begin{gather}
\alpha^{2(L-1)}(t)<L\phi^{2},\quad\alpha^{2(L-1)}(t)\beta^{2}(t)<2\phi^{2},\label{eqn:induc_weight}\\
\cR(t)\le{\left(1-\frac{\eta}{2}\right)}^{t}\cR(0).\label{eqn:induc_loss}
\end{gather}
Recall that $\phi=\max\left\{ 2\|\Phi\|_{F},\frac{e}{\sqrt L},1\right\} $.

The statements hold for $t=0$ since $\alpha(0)=1$ and $\beta(0)=0$.
Assume that the statements hold for $0,1,\dots,t$, now consider $t+1$.

From the induction assumption, $\cR(t)\le\cR(0)=\phi^{2}/8$, then
$\eta\le{\left(144L^{2}\phi^{4}\right)}^{-1}<{\left(64L^{2}\phi^{3}\sqrt{\cR(t)}\right)}^{-1}$
satisfies the requirement of Lemma~\ref{lem:loss_step}. So \eqref{eqn:induc_loss}
holds for $t+1$. Furthermore,
\[
\left\|W_{L:1}(t+1)\right\|_{2}\le\left\|W_{L:1}(t+1)\right\|_{F}\le\|\Phi\|_{F}+\sqrt{2\cR(t+1)}\le\|\Phi\|_{F}+\sqrt{2\cR(0)}\le\phi.
\]

The invariant matrices $D_{l}(0)=0$, $l=1,\dots,L-2$ and $I+D_{L-1}(0)=0$
for initialization. From Lemma~\ref{lem:invar_step}, each update
\[
\|D_{l}(s+1)-D_{l}(s)\|_{2}\le8\eta^{2}\phi^{2}\cR(s)
\]
for $l=1,\dots,L-2$ and $s=0,1,\dots,t$. From the induction assumption,
\[
  \sum_{s=0}^{t}\cR(s)\le\cR(0)\sum_{s=0}^{t}{\left(1-\frac{\eta}{2}\right)}^{s}\le\frac{2}{\eta}\cR(0)=\frac{\phi^{2}}{4\eta},
\]
then
\[
\left\|D_{l}(t+1)\right\|_{2}\le\sum_{s=0}^{t}\|D_{l}(s+1)-D_{l}(s)\|_{2}\le8\eta^{2}\phi^{2}\sum_{s=0}^{t}\cR(s)\le2\eta\phi^{4}\le\frac{1}{2L^{3}\phi^{2}}
\]
since $\eta\le{\left(4L^{3}\phi^{6}\right)}^{-1}$. Similarly, $\|I+D_{L-1}(t+1)\|_{2}\le4\eta(L+2)\phi^{2}<{\left( 4 L^2 \right)}^{-1}$.
Now from Lemma~\ref{lem:weight_bound}, the statement \eqref{eqn:induc_weight}
holds for $t+1$. Then we complete the induction.
\end{proof}

\end{document}